\documentclass[letterpaper,10pt,conference,twocolumn]{ieeeconf} 
\IEEEoverridecommandlockouts

\usepackage{verbatim}
\usepackage[pdftex]{graphicx}
\usepackage{epsfig,subfigure}
\usepackage{amsmath,amssymb,amsfonts,bm,amscd} 
\usepackage{mathrsfs}
\usepackage{setspace}
\usepackage{fancyhdr}
\usepackage{algorithm} 
\usepackage{psfrag}
\usepackage{cite}

\newtheorem{theorem}{Theorem}[section]

\newcommand{\RR}{{\mathbb R}}

\newcommand{\cD}{{\mathcal D}}

\title{\bf Multi-Robot Transfer Learning: A Dynamical System Perspective 
\thanks{The authors are with the Dynamic Systems Lab (www.dynsyslab.org), Institute
for Aerospace Studies, University of Toronto, Canada. 
M. K. Helwa is also with the Electrical Power and Machines Department, Cairo University, Giza, Egypt.
E-mail:
mohamed.helwa@robotics.utias.utoronto.ca, schoellig@utias.utoronto.ca.
This research was supported by NSERC grant RGPIN-2014-04634 and OCE/SOSCIP TalentEdge Project \#27901.}} 
\author{Mohamed~K.~Helwa and Angela P. Schoellig}

\usepackage{fancyhdr}
\fancypagestyle{specialfooter}{%
\fancyhf{}
\fancyfoot[L]{Some test text}
\fancyfoot[C]{\thepage}
}

\newenvironment{copyrightnoticeFont}{\fontsize{7pt}{8pt}\selectfont\fontfamily{phv}\selectfont}{\par}

\begin{document}
\maketitle
\thispagestyle{fancyplain}
\renewcommand{\headrulewidth}{0pt}
\pagenumbering{gobble}
\lfoot{\begin{copyrightnoticeFont}\vspace{-2em}
\textbf{Accepted final version}. To appear in \textit{Proc. of the 2017 IEEE/RSJ International Conference on Intelligent Robots and Systems}.\\
\copyright2017 IEEE. Personal use of this material is permitted. Permission from IEEE must be obtained for all other uses, in any current or future media, including reprinting/republishing this material for advertising or promotional purposes, creating new collective works, for resale or redistribution to servers or lists, or reuse of any copyrighted component of this work in other works.\end{copyrightnoticeFont}}
\pagestyle{empty}

\begin{abstract}
Multi-robot transfer learning allows a robot to use data generated by a second, similar robot to improve its own behavior. The potential advantages are reducing the time of training and the unavoidable risks that exist during the training phase. Transfer learning algorithms aim to find an optimal transfer map between different robots. In this paper, we investigate, through a theoretical study of single-input single-output (SISO) systems, the properties of such optimal transfer maps. We first show that the optimal transfer learning map is, in general, a dynamic system. The main contribution of the paper is to provide an algorithm for determining the properties of this optimal dynamic map including its order and regressors (i.e., the variables it depends on). The proposed algorithm does not require detailed knowledge of the robots' dynamics, but relies on basic system properties easily obtainable through simple experimental tests. We validate the proposed algorithm experimentally through an example of transfer learning between two different quadrotor platforms. Experimental results show that an optimal dynamic map, with correct properties obtained from our proposed algorithm, achieves 60-70\% reduction of transfer learning error compared to the cases when the data is directly transferred or transferred using an optimal static map.

\end{abstract}

\section{Introduction}
\label{sec:introd}
Machine learning approaches have been successfully applied to a wide range of robotic applications. This includes the use of regression models, e.g., Gaussian processes and deep neural networks, to approximate kinematic/dynamic models \cite{Nguyen11}, inverse dynamic models \cite{DNN16}, and unknown disturbance models \cite{Ostafew16,Felix15} of robots.
It also includes the use of reinforcement learning (RL) methods to automate a variety of human-like tasks such as screwing bottle caps onto bottles and arranging lego blocks~\cite{abbeel}. Nevertheless, machine learning methods typically require collecting   
a considerable amount of data from real-world operation or simulations of the robots, or a combination of both \cite{ICRA17}.  

Transfer learning (TL) reduces the burden of a robot to collect real-world data by enabling it to use the data generated by a second, similar robot \cite{abbeel2,Kaizad}. This is typically carried out in two phases \cite{Kaizad}. In the first phase, both robots generate data, and an optimal mapping between the generated data sets is learned. In the second phase, the learned map is used to transfer subsequent learning data collected by the second robotic system, called the source system, to the first robotic system, called the target system (see Figure \ref{fig:intro}). The goal of transfer learning is to  reduce the time needed for teaching robots new skills and to reduce the unavoidable risks that usually exist in the training phase, particularly for the cases where the target robotic platform is more expensive or more hazardous to operate than the source robotic platform.

\begin{figure}[t]
\begin{center}
\includegraphics[scale=.27, trim = 10mm 35mm 10mm 25mm]{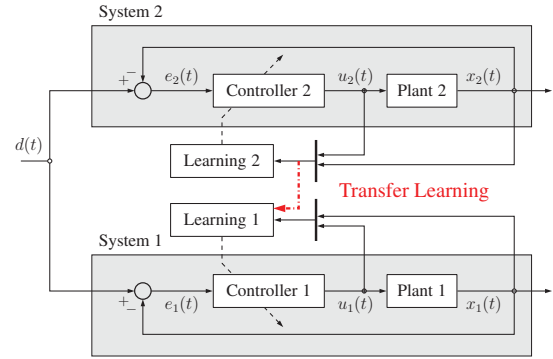} 
\end{center}
\caption{Multi-robot transfer learning framework; transfer learning allows System 1 (target system) to use data from System 2 (source system). In this paper, we provide an algorithm for determining the properties of the optimal transfer learning map between robotic systems. Figure adopted from \cite{Kaizad}.}
\label{fig:intro}
\end{figure}

The use of transfer learning in robotics can be classified into \emph{(i)} multi-task transfer learning, in which the data gathered by a robot when learning a particular task is utilized to speed up the learning of the same robot in other similar tasks \cite{Janssens2012,Hamer2013,Um2014,abbeel,abbeel2,DNN16}, and \emph{(ii)} multi-robot transfer learning, where the data gathered by a robot is used by other similar robots \cite{Schoellig12,Balaraman2010, Boutsioukis2012,Bocsi13,Manifold2,Kaizad,abbeel2,abbeel3}. The latter is the main focus of this paper. Multi-robot transfer learning has received less attention in the literature, cf. \cite{Tuyls12}. In \cite{Schoellig12}, task-dependent disturbance estimates are shared among similar robots to speed up learning in an iterative learning control~(ILC) framework, while in \cite{Balaraman2010,Boutsioukis2012}, polices and rules are transferred between simple, finite-state systems in an RL framework to accelerate robot learning. 
For a similar configuration, skills learned by two different agents in \cite{abbeel3} are used to train invariant feature spaces instead of transferring policies.
One typical approach for transfer learning, used in many applications including robotics, is manifold alignment, which aims to find an optimal, static transformation that aligns datasets \cite{Wang08,Wang09,Bocsi13,Manifold2}. In \cite{Bocsi13,Manifold2}, manifold alignment is used to transfer input-output data of a robotic arm to another arm to improve the learning of a model of the second arm. 

As partially stated in \cite{Kaizad}, although multi-robot transfer learning has been successfully applied in some robotic examples, there is still an urgent need for a general, theoretical study of when multi-robot transfer learning is beneficial, how the dynamics of the considered robots affect the quality of transfer learning, what form the optimal transfer map takes, and how to efficiently identify the transfer map from a few experiments. To fill this gap, the authors of \cite{Kaizad} recently initiated a study along these lines for two first-order, linear time-invariant (LTI), single-input single-output (SISO) systems. In particular, in \cite{Kaizad}, a simple, constant scalar is applied to align the output of the source system with the output of the target system, and then an upper bound on the Euclidean norm of the transformation error is derived and minimized with respect to (w.r.t.) the transformation parameter. The paper \cite{Kaizad} also utilizes the derived, minimized upper bound to analyze the effect of the dynamics of the source and target systems on the quality of transfer learning. 

In this paper, we study how the dynamical properties of the two robotic systems affect the choice of the optimal transfer map. This paper generalizes \cite{Kaizad}, as we consider higher-order, possibly nonlinear dynamical systems and remove the restriction that the transformation map is a static gain. 

The contributions of this paper may be summarized as follows. First, while many transfer learning methods in the literature depend on finding an optimal, static map between multi-robot data sets \cite{Bocsi13,Manifold2,Kaizad}, we show through our theoretical study that the optimal transfer map is, in general, a dynamic system. Recall that in the time domain, static maps are represented by algebraic equations, while dynamic maps are represented by differential or difference equations. Second, we utilize our theoretical study to provide insights into the correct features or properties of this optimal, dynamic map, including its order and regressors (i.e., the variables it depends on). Third, based on these insights, we provide an algorithm for selecting the correct features of this transformation map from basic properties of the source and target systems that can be obtained from few, easy-to-execute experimental tests.
Knowing these features greatly facilitates learning the map efficiently and from little data.
Fourth, we verify the soundness of the proposed algorithm 
experimentally for transfer learning between two different quadrotor platforms. Experimental results show that an optimal, dynamic map, with correct features obtained from our proposed algorithm, achieves $60$-$70\%$ reduction of transfer learning error, compared to the cases when the data is directly transferred or transferred through a static map.

This paper is organized as follows. Section \ref{sec:background} provides preliminary, dynamic-systems definitions. 
In Section \ref{sec:problem}, we define the transfer learning problem studied in this paper. In Section \ref{sec:basic}, we provide theoretical results on transformation maps that achieve perfect transfer learning, and then utilize these results to provide insights into the correct features of optimal transfer maps. In Section \ref{sec:alg}, we present our proposed, practical algorithm. Section \ref{sec:examples} includes a robotic application, and Section \ref{sec:con} concludes the paper.   
   

\section{Background}
\label{sec:background}
In this section, we review basic definitions from control systems theory needed in later sections, see \cite{chnon,Isidori}. We first introduce these definitions for linear systems, and then generalize them to an important class of nonlinear systems, namely control affine systems. To that end, consider first the LTI, SISO, $n$-dimensional state space model,
\begin{equation}
\label{eq:lin_model}
\begin{split}
\dot{x}(t)&=Ax(t) + Bu(t)\\
y(t)&=Cx(t),
\end{split}
\end{equation} 
where $x(t)\in \RR^n$ is the system state vector, $u(t)\in \RR$ is its input, and $y(t)\in \RR$ is its output. It is well known that the input-output representation of \eqref{eq:lin_model} is the transfer function
\begin{equation}
\label{eq:lin_tf}
G(s)=\frac{Y(s)}{U(s)}=C(sI-A)^{-1}B=:\frac{N(s)}{D(s)},
\end{equation} 
where N(s) and D(s) are polynomials in $s$, and we assume without loss of generality (w.l.o.g.) that they do not have common factors. 
Evidently, the system \eqref{eq:lin_tf} is bounded-input-bounded-output (BIBO) stable if and only if all the roots of $D(s)$ are in the open left half plane (OLHP). 
The \emph{relative degree} of the system \eqref{eq:lin_tf} is $deg(D(s))-deg(N(s))$, that is the order of the denominator polynomial $D(s)$ minus the order of the numerator polynomial $N(s)$. 
The definition of relative degree remains the same for discrete-time linear systems $\frac{Y(z)}{U(z)}=\frac{N(z)}{D(z)}$, where $z$ is the forward shift operator.
For \eqref{eq:lin_model}, it can be shown using the series expansion of the transfer function \eqref{eq:lin_tf} that the relative degree is the smallest integer $r$ for which $CA^{r-1}B\neq 0$, and consequently, the relative degree is also the lowest-order derivative of the output that explicitly depends on the input recalling $y^{(r)}(t)=CA^{r}x(t)+CA^{r-1}Bu(t)$, where $y^{(r)}(t)$ represents the $r$-th derivative of $y(t)$ w.r.t. $t$. 

The relative degree $r$ can be calculated from the step response of the system. For continuous-time systems, it is the lowest-order derivative of the step response $y$ that changes suddenly when the input $u$ is suddenly changed. For discrete-time systems, it is the number of sample delays between changing the input and seeing the change in the output. 

We now extend the relative degree definition to nonlinear systems. Let $C^\infty$ denote the class of smooth functions whose partial derivatives of any order exist and are continuous. 
The \emph{Lie derivative} of a smooth function $\lambda(x)$ w.r.t. a smooth vector field $f(x)$, denoted $L_f\lambda$, is the derivative of $\lambda$ in the direction of $f$; that is, $L_f\lambda:=\frac{\partial \lambda}{\partial x}f(x)$. The notation $L_{f}^2\lambda$ is used for the repeated Lie derivative; that is, $L_{f}^2\lambda=L_f(L_f\lambda(x))=\frac{\partial L_f\lambda(x)}{\partial x}f(x)$. Similarly, one can derive an expression for $L_{f}^k\lambda$, where $k>1$. 
Now consider the SISO control affine system,
\begin{equation}
\label{eq:nonlin_model}
\begin{split}
\dot{x}(t)&= f(x(t)) + g(x(t))u(t)\\
y(t)&=h(x(t)),
\end{split}
\end{equation} 
where $x(t)\in D \subset \RR^n$, $u(t)\in \RR$, $y(t)\in \RR$, and $f,~g,~h$ are $C^\infty$, nonlinear functions. Analogous to linear systems, the \emph{relative degree} of the system \eqref{eq:nonlin_model} is the smallest integer $r$ for which 
$L_gL_f^{r-1}h(x)\neq 0$, for all $x$ in the neighborhood of the operating point $x_0$. 
By successive derivatives of the output $y$, it can be shown that $y^{(r)}=L_f^r h(x)+L_gL_f^{r-1}h(x)u$.
Hence, the relative degree again represents the lowest-order derivative of the output that explicitly depends on the input. For example, the nonlinear dynamics $\ddot{\theta}=-\cos(\theta)+u$, $\theta\in (-\pi,\pi)$, with output $\theta$ and input $u$, have relative degree $2$ for all $\theta$ in the operating range. 

We next review the left inverse of the dynamics \eqref{eq:nonlin_model}, which is used in the literature to reconstruct input $u$ from the output $y$, see \cite{chnon}, and which we utilize in our discussion in Section~\ref{sec:basic}. Note that the inverse dynamics of linear systems can be easily derived from the transfer function, and its stability is determined by the zeros of the original transfer function (the roots of the polynomial $N(s)$). Suppose that~\eqref{eq:nonlin_model} has a well-defined relative degree $r$ in the operating range. Recall that $y^{(r)}=L_f^r h(x)+L_gL_f^{r-1}h(x)u$, where by definition $L_gL_f^{r-1}h(x) \neq 0$. By reordering this equation and from \eqref{eq:nonlin_model}, we obtain the inverse dynamics
 \begin{equation}
\label{eq:nonlin_inv}
\begin{split}
\dot{x}&= \left(f(x)-g(x)\frac{L_f^rh(x)}{L_gL_f^{r-1}h(x)}\right) + \frac{g(x)}{L_gL_f^{r-1}h(x)}y^{(r)}\\
u&=-\frac{L_f^rh(x)}{L_gL_f^{r-1}h(x)}+ \frac{1}{L_gL_f^{r-1}h(x)}y^{(r)},
\end{split}
\end{equation}
with input $y^{(r)}$ and output $u$. A necessary condition for the stability of~\eqref{eq:nonlin_inv} is that the dynamics of~\eqref{eq:nonlin_inv} when $y(t)=0$ uniformly (consequently, $y^{(r)}(t)=0$) are stable in the Lyapunov sense; this is called the zero dynamics of the system~\eqref{eq:nonlin_model}.  
While it appears from \eqref{eq:nonlin_inv} that the inverse dynamics have $n$ states, this is not the minimum realization of the inverse dynamics. Instead, for dynamic systems \eqref{eq:nonlin_model} with well-defined relative degree, one can always find a nonlinear coordinate transformation to convert \eqref{eq:nonlin_model} into a special form, called the Byrnes-Isidori normal form. Using this form, a minimum realization of \eqref{eq:nonlin_inv} can be derived with $(n-r)$ states, inputs $y,\dot{y},\cdots,y^{(r)}$, and output $u$, refer to \cite{Isidori,chnon}. 

\section{Problem Statement}
\label{sec:problem}
In this paper, we study transfer learning between two robotic systems from a dynamical system perspective. We use our theoretical results to provide insights into the properties of optimal transfer maps. These insights facilitate the identification of this optimal map from data using, for instance, system identification algorithms. 
In particular, as shown in Figure \ref{fig:problem}, we consider a source SISO dynamical system $\cD_S$ with input reference signal $d$ and output $y_s$, representing the source robot, and a target SISO dynamical system $\cD_T$ with the same input $d$ and output $y_t$, representing the target robot. Assuming that $d$ is an arbitrary bounded signal, the \emph{transfer learning problem} is to find a transfer map $\cD_{TL}$ with input $y_s$ and output $y_{TL}$ such that \emph{the error~$e$ between $y_t$ and $y_{TL}$ is minimized.}   

To make the transfer learning problem tractable, we assume that both the source system $\cD_S$ and the target system $\cD_T$ are input-output stable (this is typically characterized by the BIBO stability notion for LTI systems and by the Input to Output Stability (IOS) notion for nonlinear systems~\cite{Sontag}). This is a reasonable assumption, given that input-output stability is necessary for the safe operation of the robot and transfer learning is only efficient for stable systems~\cite{Kaizad}.   
\begin{figure}[t]
\begin{center}
\includegraphics[scale=.47, trim = 20mm 10mm 10mm 0mm]{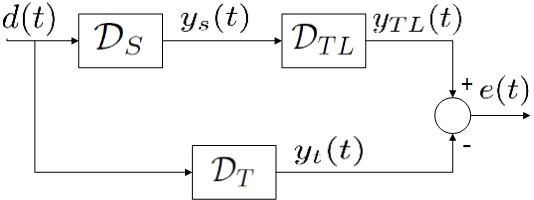}
\end{center}
\caption{Illustrative figure of the transfer learning problem: the objective is to identify a transfer learning map $\cD_{TL}$ to minimize the error between the transferred output $y_{TL}$ and the target system's output $y_t$.}
\label{fig:problem}
\end{figure}

\section{Main Results}
\label{sec:basic}
In this section, we assume that the source system dynamics $\cD_S$ and the target system dynamics $\cD_T$ are known, and then provide a theoretical study on when it is possible to identify a dynamic map $\cD_{TL}$ that achieves perfect transfer learning from $\cD_S$ to $\cD_T$, i.e., it perfectly aligns $y_t$ and $y_{TL}$ resulting in zero transfer learning error ($e(t)=0$). From this theoretical study, we provide insights into the correct properties of the dynamic map $\cD_{TL}$, including its order, relative degree, and input-output variables. We then show that these properties can be determined from basic properties of the source and target systems, which can be identified through short, simple experiments. There is no need to know the source/target system dynamics a priori. 
Knowing the properties of the optimal transfer map greatly facilitates the identification of this map from data using standard system identification tools, as we will show in Section \ref{sec:examples}.   

 For simplicity, we first present our theoretical study and insights for linear systems, and then show that these insights remain valid for nonlinear systems. To that end, in this paper, we say that an LTI system is \emph{minimum-phase} if the dynamics of the system and its inverse dynamics are BIBO stable. 
\begin{theorem}
\label{thm_rel_lin}
Consider two continuous-time, BIBO stable, SISO, LTI systems, with rational transfer functions $G_S(s)$ and $G_T(s)$, and suppose that $G_S(s)$ is minimum-phase. Then, there exists a causal, BIBO stable map from the source system $G_S(s)$ to the target system $G_T(s)$ that achieves perfect transfer learning if and only if the relative degree of $G_S(s)$ $\leq$ the relative degree of $G_T(s)$. 
\end{theorem}
\begin{proof}
($\Rightarrow$) Let $G_S(s):=\frac{N_S(s)}{D_S(s)}$ and $G_T(s):=\frac{N_T(s)}{D_T(s)}$. By assumption, there exists a causal function $G_\alpha(s)$ such that for any bounded input $u$, $(G_\alpha(s)G_S(s)-G_T(s))U(s)=0$. Since $u$ is arbitrary, then clearly $G_\alpha(s)G_S(s)-G_T(s)=0$. Equivalently, $G_\alpha(s)G_S(s) = G_T(s)$. Let $G_\alpha(s):=\frac{N_\alpha(s)}{D_\alpha(s)}$. Then, we have
$\frac{N_\alpha(s)}{D_\alpha(s)} \frac{N_S(s)}{D_S(s)} = \frac{N_T(s)}{D_T(s)}$.
Even in the presence of pole-zero cancellations, it can be shown that the above equation implies  
\begin{equation*}
\begin{split}
(deg(D_\alpha(s))+deg(D_S(s)))-(deg(N_\alpha(s))+deg(N_S(s)))\\ =(deg(D_T(s))-deg(N_T(s))).
\end{split}
\end{equation*}
By reordering the terms on the left hand side (LHS), the summation of the relative degrees of $G_\alpha(s)$ and $G_S(s)$ is equal to the relative degree of $G_T(s)$. Since $G_\alpha(s)$ is causal, the relative degree of $G_\alpha(s)\geq0$, and the result follows.

($\Leftarrow$) Suppose that the relative degree of $G_S(s)$ $\leq$ the relative degree of $G_T(s)$. We construct a causal, stable map $G_\alpha(s)$ that achieves perfect transfer learning. Let
\begin{equation}
\label{eq:opt_TL}
G_\alpha(s):=\frac{D_S(s)}{N_S(s)} \frac{N_T(s)}{D_T(s)}.
\end{equation}
Since the relative degree of $G_S(s)$ $\leq$ the relative degree of $G_T(s)$, we have 
$deg(D_S(s))-deg(N_S(s))\leq deg(D_T(s))-deg(N_T(s))$. Equivalently, 
$deg(D_S(s))+deg(N_T(s))\leq deg(D_T(s))+deg(N_S(s))$.
This implies $G_\alpha(s)$ is a causal function. Notice that the poles of $G_\alpha(s)$ are a subset of the roots of $D_T(s)$ and $N_S(s)$. Then, since 
$G_T(s)$ is BIBO stable and $G_S(s)$ is minimum-phase by assumption, 
the roots of $D_T(s)$ and $N_S(s)$ are all in the OLHP, and $G_\alpha(s)$ is a BIBO stable transfer function. Next, one can verify that for the selected $G_\alpha(s)$, we have $G_\alpha(s)G_S(s)-G_T(s)=0$, and consequently $G_\alpha(s)$ achieves perfect transfer learning. 
\end{proof}  

Equation \eqref{eq:opt_TL} and its associated discussion are similar to standard methods in linear control synthesis. Similar results can be derived for discrete-time LTI systems. 
\begin{theorem}
\label{thm_rel_lin_dis}
Consider two discrete-time, BIBO stable, SISO, LTI systems, with rational transfer functions $G_S(z)$, $G_T(z)$, and suppose that $G_S(z)$ is minimum-phase. Then, there exists a causal, BIBO stable map from the source system $G_S(z)$ to the target system $G_T(z)$ that achieves perfect transfer learning if and only if the relative degree of $G_S(z)$ $\leq$ the relative degree of $G_T(z)$.
\end{theorem}  

Theorems \ref{thm_rel_lin} and \ref{thm_rel_lin_dis} provide the following insights into the properties of the optimal transfer maps between systems. 

{\bf Insight 1:} From \eqref{eq:opt_TL}, one can see that the optimal transfer map is, in general, a dynamic system. Therefore, limiting the transfer map to be static \cite{Kaizad,Bocsi13} may be restrictive; see also Section \ref{sec:examples}. 

{\bf Insight 2:} To be able to identify the optimal transfer learning map from data using system identification algorithms, it is important to decide on the right order of the dynamic map. From~\eqref{eq:opt_TL}, the order of the optimal map that achieves zero transfer learning error is in general $deg(N_S(s))+deg(D_T(s))$. Equivalently, the correct order of the map is $n_s-r_s+n_t$, where $n_s$ is the order of the source system, $n_t$ is the order of the target system, and $r_s$ is the relative degree of the source system, which can be identified experimentally from the system step response as stated in Section \ref{sec:background}.  

{\bf Insight 3:} From \eqref{eq:opt_TL}, the relative degree of the optimal transfer learning map is $r_t-r_s$, where $r_s,~r_t$ are the relative degrees of the source and target systems, respectively. The relative degree of the transfer map is also needed for standard system identification algorithms. By knowing the order and the relative degree of the transfer learning map, the regressors of the map are determined. For instance, for a discrete-time transfer learning map of order~$3$ and relative degree~$1$, the map relates the output $y_{TL}(k)$ to the inputs $y_{TL}(k-1),~y_{TL}(k-2),~y_{TL}(k-3),y_s(k-1),y_s(k-2),y_s(k-3)$.   

{\bf Insight 4:} From Theorems \ref{thm_rel_lin} and \ref{thm_rel_lin_dis}, if the relative degree of the source system $r_s$ is greater than the relative degree of the target system $r_t$, then we cannot find a causal map satisfying perfect transfer learning (zero transfer learning error). Nevertheless, since we have the complete input-output data of the source robot available before carrying out the transfer learning, the causality requirement can be relaxed. For instance, although a discrete-time transfer learning map from $\{y_s(k-1),y_s(k),y_s(k+1),y_{TL}(k-1)\}$ to $y_{TL}(k)$ is non-causal, it can be implemented since all the future values of $y_s$ are saved before using the transfer learning map for transferring the source data to the target system. 
However, system identification computer tools such as MATLAB's identification toolbox are typically used for identifying causal models such as causal transfer functions (MATLAB: \texttt{tfest}), nonlinear autoregressive exogenous (NARX) models (MATLAB: \texttt{nlarx}), and recurrent neural networks, among others. One possible trick to solve this problem is to tailor the input of the dynamic transfer learning map as follows. 
First, for continuous-time systems, we know from \eqref{eq:opt_TL} that the optimal transfer map is $\frac{Y_{TL}(s)}{Y_s(s)}=\frac{N_\alpha(s)}{D_\alpha(s)}$, where $deg(D_\alpha)=n_s+n_t-r_s$, $deg(N_\alpha)=n_s+n_t-r_t$, and for this case $deg(N_\alpha)>deg(D_\alpha)$ (non-causal map). Instead of identifying this non-causal map, we use standard system identification computer tools to identify the causal map $\frac{Y_{TL}(s)}{s^{(r_s-r_t)}Y_s(s)}=\frac{N_\alpha(s)}{s^{(r_s-r_t)}D_\alpha(s)}$, which represents the Laplace transform of the map from $y_s^{(r_s-r_t)}(t)$, the $(r_s-r_t)$-th derivative of the source robot's output $y_s(t)$, to $y_{TL}(t)$. Hence, from the $y_s(t)$ response, we calculate $y_s^{(r_s-r_t)}(t)$ (and possibly low-pass filter $y_s(t)$ to avoid noise amplification). We then use $y_s^{(r_s-r_t)}(t)$ as the input to the dynamic transfer map to be identified.
Notice that this is not the only choice. One can, for example, use the system identification tools to identify the causal map $\frac{Y_{TL}(s)}{P(s)Y_s(s)}=\frac{N_\alpha(s)}{P(s)D_\alpha(s)}$, where $P(s)$ is a known $(r_s-r_t)$-th order polynomial in $s$, with all its roots in the OLHP. Since both $y_s$ and the polynomial $P(s)$ are known, one can define the data column for the tailored input of the dynamic map to be identified. For instance, if for $r_s-r_t=1$, one selects the polynomial $P(s)=s+1$, then the tailored input to the dynamic map, to be identified, is $y_s(t)+\dot{y}_s(t)$, and so on. 
Similarly, for discrete-time systems, we use system identification tools to identify the causal map $\frac{Y_{TL}(z)}{z^{(r_s-r_t)}Y_s(z)}=\frac{N_\alpha(z)}{z^{(r_s-r_t)}D_\alpha(z)}$. For this tailored, causal map, the input is $y_{s,mod}$, which is obtained by shifting each element in the data column for $y_s$ forward in time by $(r_s-r_t)$ samples, and the output is $y_{TL}$.   
%

We now show that these insights remain valid for {\it nonlinear systems}. Theorems  \ref{thm_rel_lin}, \ref{thm_rel_lin_dis} and the related insights mainly depend on the definition of relative degree, which is also defined for control affine nonlinear systems as discussed in Section \ref{sec:background}. Hence, suppose that we have two smooth, control affine nonlinear systems of the form \eqref{eq:nonlin_model}: a source system with order $n_s$ and a well-defined relative degree $r_s$ in the operating range, and a target system with order $n_t$ and a well-defined relative degree $r_t$ in the operating range. Also, suppose that the source system dynamics and its inverse dynamics are both input-output stable, and that the target dynamics are input-output stable \cite{Sontag}.  From \eqref{eq:opt_TL}, one can see that the optimal transfer map, that achieves zero transfer learning error, is composed of two cascaded systems: the inverse of the source system dynamics and the target system dynamics. Intuitively, the inverse of the source dynamics is utilized to successfully reconstruct the input $d$ from the source output response $y_s$, and then the target system dynamics are applied to exactly obtain the target output response $y_t$ from $d$. A similar approach can be utilized for nonlinear systems to get zero transfer learning error. From the last paragraph in Section \ref{sec:background}, we know that the minimum realization of the inverse dynamics of the source system has order $n_s-r_s$, while the order of the target system dynamics is by definition $n_t$. Therefore, the correct order of the optimal dynamic map, composed of these two cascaded systems, is $n_s+n_t-r_s$, which is the same conclusion we reached in Insight 2 for linear systems. Then, for this optimal map, we have from the relative degree definition for the source system with internal state $x$ (subscript $S$ is dropped from $f$, $g$, $h$ for notational simplicity)
$y_s^{(r_s)}=L_f^{r_s} h(x)+L_gL_f^{r_s-1}h(x)d$, where $L_gL_f^{r_s-1}h(x)\neq 0$, and for the target dynamics with internal state $v$ (subscript $T$ is dropped from $f$, $g$, $h$) 
$y_{TL}^{(r_t)}=L_f^{r_t} h(v)+L_gL_f^{r_t-1}h(v)d$, where $L_gL_f^{r_t-1}h(v)\neq 0$.
By getting $d$ from the first equation and substituting it in the second one, we have 
$y_{TL}^{(r_t)}=L_f^{r_t} h(v)-L_gL_f^{r_t-1}h(v) \frac{L_f^{r_s}h(x)}{L_gL_f^{r_s-1}h(x)}+ \frac{L_gL_f^{r_t-1}h(v)}{L_gL_f^{r_s-1}h(x)}y_s^{(r_s)}$, 
i.e., $y_{TL}^{(r_t)}$ explicitly depends on $y_s^{(r_s)}$, and consequently, it is reasonable to select the relative degree of the optimal map from $y_s$ to $y_{TL}$ to be $r_t-r_s$ as in Insight 3. 

\begin{algorithm}[t]
\caption{Finding the properties of transfer learning maps}
\label{alg_paper1}
~~\\
{\bf Given:} (1) Two SISO robotic systems: a source system with order $n_s$ and well-defined relative degree $r_s$, and a target system with order $n_t$ and well-defined relative degree $r_t$; (2) output responses of the source and target systems, $y_s$ and $y_t$, respectively, for a bounded reference input $d$.\\
{\bf Objective:} Find the properties of the optimal, dynamic transfer learning map, particularly its order, relative degree, input and output training data. \\  
{\bf Steps:}
\begin{enumerate}
\item If $r_s \leq r_t$, proceed. Otherwise, jump to step 3. 
\item Input data to the dynamic map: $y_s$; output data: $y_t$; order of the dynamic map is $n_s+n_t-r_s$; relative degree of the dynamic map is $r_t-r_s$. Stop.
\item Define tailored input data to the dynamic map $y_{s,mod}$. For continuous-time transfer learning maps, $y_{s,mod}=y_s^{(r_s-r_t)}$, the $(r_s-r_t)$-th derivative of the saved time response $y_s$. For discrete-time maps, the $y_{s,mod}$ data column is obtained from the $y_s$ column by shifting each element of $y_s$ forward in time by $r_s-r_t$ samples.
\item Input data: $y_{s,mod}$; output data: $y_t$; order of the dynamic map is $n_s+n_t-r_t$; relative degree of the dynamic map is zero.  
\end{enumerate} 
\end{algorithm}
\section{Algorithm}
\label{sec:alg}
Inspired by the insights presented in the previous section, we provide an algorithm for getting the correct properties of the optimal, dynamic transfer learning map between two robotic systems from simple experiments.
As discussed before, we assume that both systems are input-output stable and that the source system has stable inverse dynamics. 
Once the properties of the map are determined, one can utilize any system identification tool, such as MATLAB's identification toolbox, to identify the map from collected data as we will show in our practical examples in Section \ref{sec:examples}. The identified map can then be used to transfer any subsequent learning data from the source system to the target system. The main steps are summarized in Algorithm~\ref{alg_paper1}. Notice that step 2 of the algorithm directly follows from Insights 2, 3 in Section \ref{sec:basic}, while steps 3, 4 directly follow from Insight 4.  

To better understand the algorithm, suppose, as a toy numerical example, that we have two minimum-phase, discrete-time systems with zero initial conditions and orders $n_s=5$ and $n_t=3$.
We assume that this is the only available information about the systems. To identify the relative degrees of the systems, we apply at time step $k=0$ a step input to both systems. From the step response of the source system, we found that the output only changes at $k=4$, and consequently, $r_s=4$. Similarly, we found that $r_t=3$. We then follow the steps of our algorithm: (1) since $r_s > r_t$, we jump to step 3; (3) we construct the $y_{s,mod}$ data column by shifting each element in the step response $y_s$ forward in time by $r_s-r_t=1$ sample; (4) the input training data is $y_{s,mod}$, the output training data is $y_t$, the map order is $5$, and its relative degree is $0$. The transfer learning map should relate $y_{TL}(k)$ to $y_{TL}(k-1),\cdots,y_{TL}(k-5)$$,y_{s,mod}(k),\cdots,y_{s,mod}(k-5)$ to best fit the output data $y_t$. 

One advantage of the proposed algorithm is that it does not require precise knowledge of the robots' dynamics and/or parameters. Instead, it only requires the knowledge of basic properties of the robotic systems, namely the system order and the relative degree. The order of the robotic system can be determined from approximate physics models, or even from general information about the robot structure.  
For instance, an $N$-link manipulator has a dynamical model of order $2N$. Similarly, the relative degree of the system may be determined from physics models, or experimentally from the step response of the system as discussed in Section \ref{sec:background}. Another advantage of the proposed algorithm is that it is generic in the sense that it can be combined with any system identification model/algorithm. For instance, one can utilize the proposed algorithm to determine the correct properties of both linear and nonlinear dynamic transfer maps.

\section{Application}
\label{sec:examples}
In this section, we utilize the proposed algorithm to identify a dynamic transfer learning map between two different quadrotor platforms, namely the Parrot AR.Drone 2.0 and the Parrot Bebop 2.0 (see Figure \ref{fig_drones}), and then verify through experimental results the effectiveness of our proposed map. 

\begin{figure}[t]
\begin{center}
\includegraphics[scale=.33, trim = 0mm 20mm 0mm 17mm]{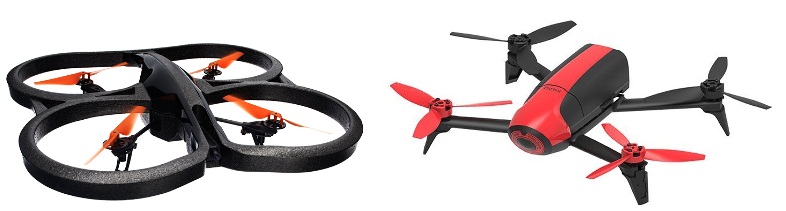} 
\end{center}
\caption{The two quadrotor platforms used in our experiments; we learn a transfer learning map from the Parrot AR.Drone 2.0 (left) to the Parrot Bebop 2.0 (right).}
\label{fig_drones}
\end{figure}  

Quadrotor vehicles have six degrees of freedom: the translational position of the vehicle's center of mass $(x,y,z)$, measured in an inertial coordinate frame, and the vehicle's attitude, represented by the Euler angles $(\phi,\theta,\psi)$, namely the roll, pitch, and yaw angles, respectively. The full state of the vehicle also includes the translational velocities $(\dot{x},\dot{y},\dot{z})$ and the rotational velocities $(p,q,r)$, resulting in a dynamic model of the vehicle with $12$ states. Detailed description of the quadrotor's dynamic model can be found in \cite{angela_quad}. In our experiments, the quadrotor's states are all measured by the overhead motion capture system, which consists of ten $4$-mega pixel cameras running at $200$\,Hz. 

In our study, the two quadrotor platforms utilize a control strategy that consists of two controllers: \emph{(i)} an on-board controller that runs at $200$\,Hz, receives the desired roll $\phi_d$, pitch $\theta_d$, yaw velocity $r_d$
and the $z$-axis velocity $\dot{z}_d$, and outputs the thrusts of the quadrotor's four motors, and \emph{(ii)} an off-board controller that is implemented using the open-source Robot Operating System (ROS), runs at $70$\,Hz, receives the desired vehicle's position, and outputs the commands $(\phi_d,\theta_d,r_d,\dot{z}_d)$ to the on-board controller. 
For the off-board controller, we utilize a nonlinear control strategy to stabilize the $z$-position of the vehicle to a fixed value and the yaw angle to zero, and then manipulate $\phi_d$ and $\theta_d$ to control the vehicle's motion in the $x$-, $y$-directions. In particular, we select $\phi_d$ and $\theta_d$ to implement a nonlinear transformation that decouples the dynamics in the $x$-, $y$-directions into approximate, linearized, second-order dynamics in each direction, and then utilize a proportional-derivative (PD) controller for each direction. More details can be found in \cite{HS16}. 

\begin{figure}[t]
\begin{center}
\includegraphics[scale=.25, trim = 10mm 20mm 10mm 20mm]{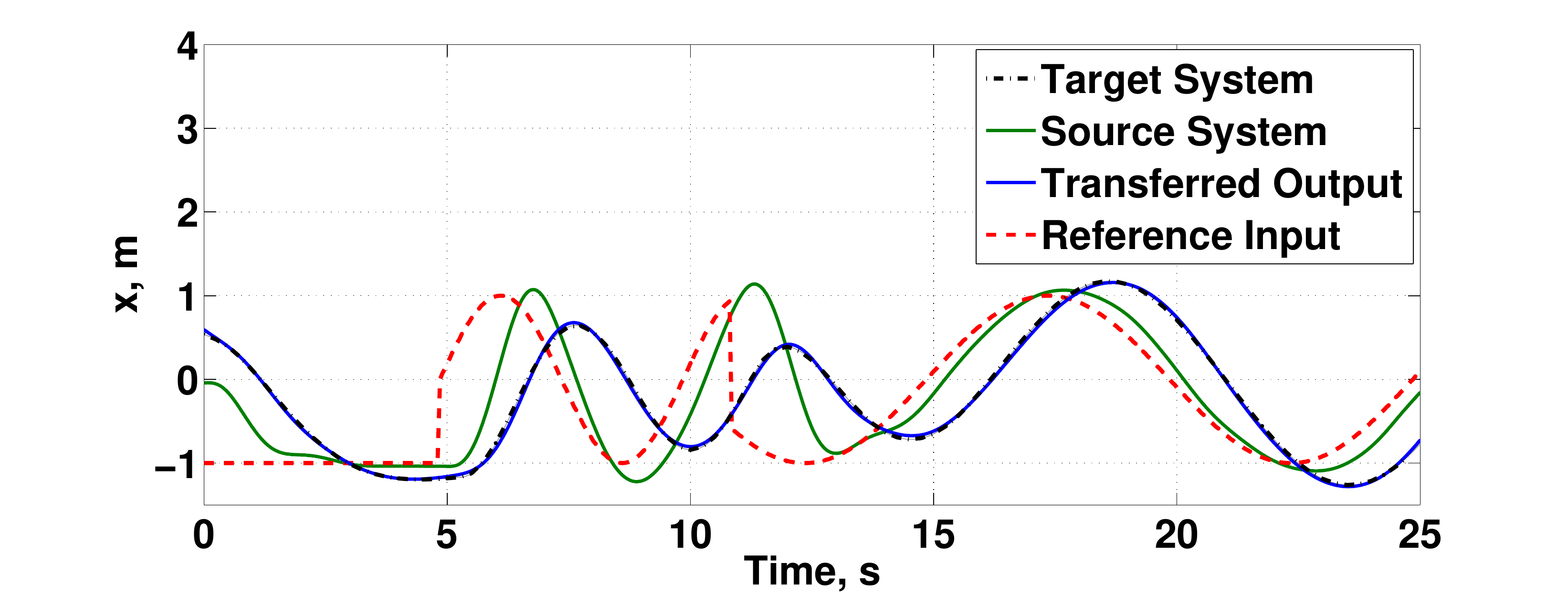} 
\end{center}
\caption{The training input-output data used to identify the transfer learning map in the $x$-direction, and the transferred output using the proposed map. The transferred output fits the target system's output with $95.79\%$.}
\vspace{-1.5em}
\label{fig:training_quadx}
\end{figure}

In this application, we identify a transfer learning map from the Parrot AR.Drone 2.0 platform, the source system, to the Parrot Bebop 2.0 platform, the target system, for each of the $x$-, $y$-directions. We first stabilize the vehicle's $y$- and $z$-positions to constant values, and study the motion in the $x$-direction. For this case, the input to each system is the desired $x$-value reference, while the output is the actual $x$-value of the quadrotor. We start by collecting data for both vehicles in the $x$-direction. In particular, we apply the same desired reference $x_d$ to both vehicles and detect their outputs (see Figure \ref{fig:training_quadx}). We then utilize this collected data to identify a continuous-time transfer learning map with the aid of our proposed algorithm. Following the previous paragraph, we know that under the applied control strategy the $x$-direction dynamics for the quadrotors have approximately order $2$, and by analyzing the dynamic equations, we have found that the relative degree for the quadrotors in this case is $1$. We have verified this value experimentally from the collected data in Figure \ref{fig:training_quadx} as discussed in Section \ref{sec:background}. To sum up, we have $n_s=n_t=2$ and $r_s=r_t=1$. By following the steps of our algorithm, the correct input to the transfer learning map is the $x$-output of the source system, $y_s$, its output is the $x$-output of the target system, $y_t$, its dynamic order is $3$, and its relative degree is $0$. Since the applied control strategy turns the closed loop into an approximately linear behavior in the $x$-direction as discussed in the previous paragraph, we identify a linear transfer learning map with the desired properties using MATLAB's \texttt{tfest} for identifying transfer functions. The obtained transfer function fits the training data ($y_t$) with $95.79\%$, measured based on the well-known normalized root mean square error (NRMSE) fitness value ($fit=100(1-NRMSE)\%$), see Figure \ref{fig:training_quadx}. For comparison, we have also identified an optimal, static gain from $y_s$ to $y_t$ using MATLAB's \texttt{tfest} with the function's orders set to $(0,0)$; the gain is 0.6925, and it fits the data with $27.28\%$. 

\begin{figure}[t]
\begin{center}
\includegraphics[scale=.24, trim = 10mm 20mm 10mm 20mm]{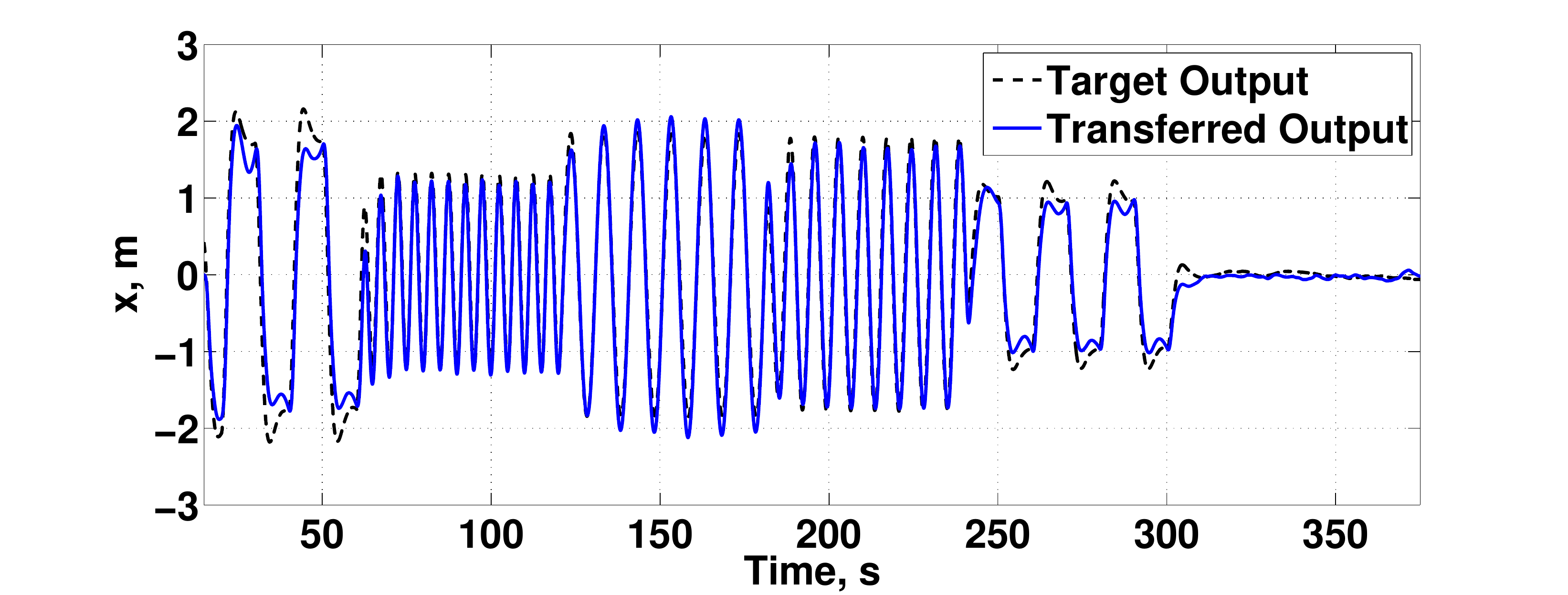}
\end{center}
\caption{The output of the target quadrotor and the transferred output using our proposed, dynamic map for transferring six minutes of collected data of the source quadrotor. The proposed map achieves an RMS error of $0.2142$~m, which reduces the direct transfer learning error by $70.65\%$.}
\label{fig:test_quadx}
\end{figure}

We next test the identified transfer learning maps for transferring six minutes of collected data from the Parrot AR.Drone 2.0 to the Parrot Bebop 2.0. Figure \ref{fig:test_quadx} shows the actual output of the target quadrotor and the transferred output using our proposed map.
The proposed map achieves an RMS error of $0.2142$\,m, compared to $0.7297$\,m for direct transfer learning (identity map), and $0.6946$\,m for the identified, optimal, static map. The proposed map achieves $70.65\%$ reduction in error over the direct transfer learning, while the optimal, static map achieves only $4.81\%$ reduction.  


\begin{figure}[t]
\begin{center}
\includegraphics[scale=.212, trim = 10mm 26mm 10mm 19.5mm]{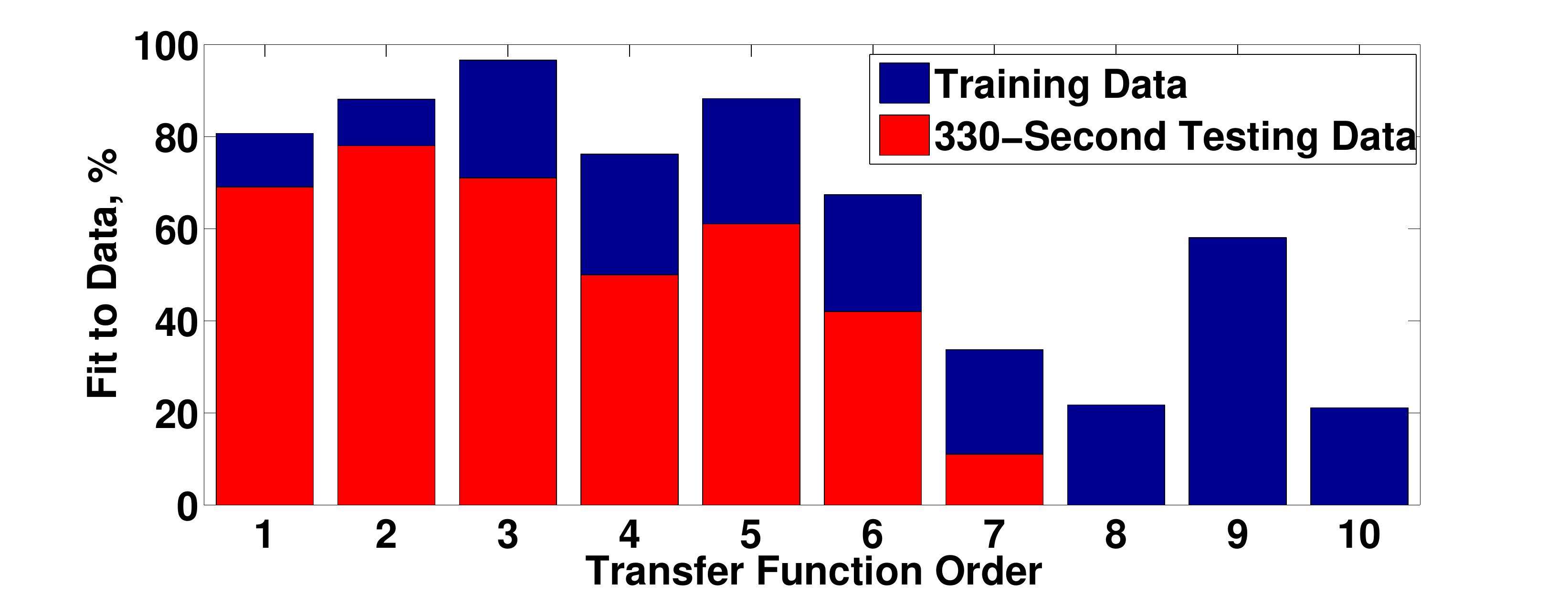} 
\end{center}
\caption{The figure shows how optimal transfer functions with different orders and zero relative degree fit the $y$-direction training and testing data. The NRMSE fitness measure is used. The dynamic order 3, proposed by our algorithm, best fits the training data, and it achieves the second highest fit to the $330$-second testing data.}
\vspace{-1.5em}
\label{fig:order_justification}
\end{figure}

We similarly identify a transfer learning map from the Parrot AR.Drone 2.0 to the Parrot Bebop 2.0 in the $y$-direction. We omit the details for brevity. 
The proposed, identified map is a transfer function with order $3$ and zero relative degree, and it fits the training data with $96.59\%$. The optimal, static TL gain is $0.518$, and it fits the data with~$21.73\%$. We then test the proposed, dynamic map for transferring $330$-second $y$-direction data from the Parrot AR.Drone 2.0 to the Parrot Bebop 2.0. 
For this testing data, our proposed map has an RMS error of $0.192$\,m, which achieves $68.29\%$ error reduction compared to direct transfer learning and $62\%$ error reduction compared to the optimal, static gain. Figure \ref{fig:order_justification} shows how optimal transfer functions with different orders and zero relative degree fit the training and testing data. The order 3, proposed by our algorithm, best fits the training data, and it achieves the second highest fit to the $330$-second testing data after the order $2$. However, using this testing data as training data for identifying new transfer functions again shows that the third-order transfer function outperforms the second-order one in fitting this testing data. 
While it is expected due to overfitting that higher-order transfer functions have lower fit on the testing data, this is less obvious for the training data. The explanation is likely that the orders in Figure \ref{fig:order_justification} are not high enough to overfit the $25$-second training data ($5000$ data points). Indeed, for order $50$, the obtained transfer function fits the training data with $98.14\%$, but it completely fails to transfer the testing data.     

\begin{table}[!t]
\small
\caption{TL error (RMS) for data from tracking a circle} 
\label{table_errors} 
\centering
\vspace*{-3.9mm}
\begin{tabular}{|c|c|c|c|}
\hline
&\textbf{Direct TL} &\textbf{Proposed Map}& \textbf{Reduction} \\ \hline
{\bf $x$-direction} & $0.87$\,m & $0.441$\,m & $49.26\%$ \\ \hline
{\bf $y$-direction} & $0.675$\,m & $0.057$\,m & $91.5\%$  \\ \hline
{\bf Total} & $1.101$\,m & $0.445$\,m & $59.58\%$  \\ \hline
\end{tabular}
\end{table}

We then test the proposed, identified TL maps in the $x$-, $y$-directions for transferring the $x$-, $y$-data from the Parrot AR.Drone 2.0 to the Parrot Bebop 2.0, for the case where both vehicles are required to track a unit circle in the $(x,y)$-plane (with frequency $0.14$\,Hz, which is different from the frequencies of the references used in the training data). Table \ref{table_errors} summarizes the transfer learning errors 
for both the proposed map and the direct transfer learning.
Our proposed, dynamic maps achieve significant reduction of the direct transfer learning errors. However, the total improvement is less than in the previous examples. This is likely due to the unmodeled coupling in the $x$-, $y$-directions. The optimal TL map for this case should be a $(2\times2)$ matrix of dynamic maps to account for the coupling between the two directions. 

\section{Conclusions}
\label{sec:con}
We have studied multi-robot transfer learning (TL) from a dynamical system perspective for SISO systems. While many existing methods utilize static TL maps, we have shown that the optimal TL map is a dynamic system and provided an algorithm for determining the properties of the dynamic map, including its order and regressors, from knowing the order and relative degree of the systems. These basic system properties can be obtained from approximate physics models of the robots or from simple experiments. 
Our results show that for the testing data, dynamic maps, with correct features from our proposed algorithm, achieve on average $66\%$ reduction of TL errors compared to direct TL, while optimal, static gains achieve only $15\%$ reduction.  
For future research, we consider the generalization of this algorithm to multi-input multi-output systems.  

\bibliographystyle{IEEEtranS}

\begin{thebibliography}{}
\providecommand{\url}[1]{#1}
\csname url@samestyle\endcsname
\providecommand{\newblock}{\relax}
\providecommand{\bibinfo}[2]{#2}
\providecommand{\BIBentrySTDinterwordspacing}{\spaceskip=0pt\relax}
\providecommand{\BIBentryALTinterwordstretchfactor}{4}
\providecommand{\BIBentryALTinterwordspacing}{\spaceskip=\fontdimen2\font plus
\BIBentryALTinterwordstretchfactor\fontdimen3\font minus
  \fontdimen4\font\relax}
\providecommand{\BIBforeignlanguage}[2]{{%
\expandafter\ifx\csname l@#1\endcsname\relax
\typeout{** WARNING: IEEEtranS.bst: No hyphenation pattern has been}%
\typeout{** loaded for the language `#1'. Using the pattern for}%
\typeout{** the default language instead.}%
\else
\language=\csname l@#1\endcsname
\fi
#2}}
\providecommand{\BIBdecl}{\relax}
\BIBdecl

\end{thebibliography}


\begin{thebibliography}{99}

\bibitem{Nguyen11}
D. Nguyen-Tuong, J. Peters. Model learning for robot
control: a survey. \emph{Cognitive Processing}, vol. 12(4), pp. 319-340, 2011.

\bibitem{DNN16}
Q. Li, J. Qian, Z. Zhu, X. Bao, M. K. Helwa, A. P. Schoellig. Deep neural networks for improved, impromptu trajectory tracking of quadrotors. \emph{IEEE Intl. Conf. on Robotics and Automation}, 2017, pp. 5183-5189.

\bibitem{Ostafew16}
C. J. Ostafew, A. P. Schoellig, T. D. Barfoot.
Robust constrained learning-based NMPC enabling reliable mobile robot path tracking.
\emph{The International Journal of Robotics Research}, vol. 35(13), pp. 1547-1563, 2016. 

\bibitem{Felix15}
F. Berkenkamp, A. P. Schoellig. Safe and robust
learning control with Gaussian processes. 
\emph{European Control Conf.}, 2015, pp. 2501-2506.

\bibitem{abbeel}
S. Levine, N. Wagener, P. Abbeel.
Learning contact-rich manipulation skills with guided policy search. \emph{IEEE Intl. Conf. on Robotics and Automation}, 2015, pp. 156-163.

\bibitem{ICRA17}
A. Marco, F. Berkenkamp, P. Hennig, A. P. Schoellig, A. Krause, S. Schaal, S. Trimpe. Virtual vs. real: trading off simulations and physical experiments
in reinforcement learning with Bayesian optimization. \emph{IEEE Intl. Conf. on Robotics and Automation}, 2017, pp. 1557-1563. 

\bibitem{abbeel2}
C. Devin, A. Gupta, T. Darrell, P. Abbeel, S. Levine.
Learning modular neural network policies for multi-task and multi-robot transfer. \emph{IEEE Intl. Conf. on Robotics and Automation}, 2017, pp. 2169-2176. 

\bibitem{abbeel3}
A. Gupta, C. Devin, Y. Liu, P. Abbeel, S. Levine.
Learning invariant feature spaces to transfer skills with reinforcement learning.
\emph{Intl. Conf. on Learning Representations}, 2017. \emph{Available at ArXiv, arXiv:1703.02949 [cs.AI]}. 

\bibitem{Kaizad}
K. V. Raimalwala, B. A. Francis, A. P. Schoellig.
An upper bound on the error of alignment-based transfer learning between two linear, time-invariant, scalar systems.
\emph{IEEE/RSJ Intl. Conf. on Intelligent Robots and Systems}, 2015, pp. 5253-5258.  

\bibitem{Janssens2012}
P. Janssens, G. Pipeleers, J. Swevers. Initialization of
ILC based on a previously learned trajectory.
\emph{American Control Conf.}, 2012, pp. 610-614.

\bibitem{Hamer2013}
M. Hamer, M. Waibel, R. D'Andrea. Knowledge
transfer for high-performance quadrocopter maneuvers. \emph{IEEE/RSJ Intl. Conf. on Intelligent
Robots and Systems}, 2013, pp. 1714-1719.

\bibitem{Um2014}
T. T. Um, M. S. Park, J.-M. Park. Independent
joint learning: a novel task-to-task transfer learning scheme for
robot models. \emph{IEEE Intl. Conf. on
Robotics and Automation}, 2014, pp. 5679-5684.

\bibitem{Schoellig12}
A. P. Schoellig, J. Alonso-Mora, R. D'Andrea.
Limited benefit of joint estimation in multi-agent iterative
learning. \emph{Asian Journal of Control}, vol. 14(3), pp. 613-623, 2012.

\bibitem{Balaraman2010}
B. Lakshmanan, R. Balaraman. Transfer learning across
heterogeneous robots with action sequence mapping. \emph{IEEE/RSJ Intl. Conf. on Intelligent Robots
and Systems}, 2010, pp. 3251-3256.

\bibitem{Boutsioukis2012}
G. Boutsioukis, I. Partalas, I. Vlahavas. Transfer
learning in multi-agent reinforcement learning domains. \emph{
Recent Advances in Reinforcement Learning}, pp. 249-260, Springer-Verlag Berlin, 2012.

\bibitem{Bocsi13}
B. B\'{o}csi, L. Csat\'{o}, J. Peters. Alignment-based transfer
learning for robot models. \emph{Intl.
Joint Conf. on Neural Networks}, 2013, pp. 1-7.

\bibitem{Manifold2}
N. Makondo, B. Rosman, O. Hasegawa. Knowledge transfer for learning robot models via local procrustes
analysis. \emph{IEEE-RAS Intl. Conf. on Humanoid Robots}, 2015, pp. 1075-1082.   

\bibitem{Tuyls12}
K. Tuyls, G. Weiss. Multiagent learning: basics, challenges, and prospects. \emph{AI Magazine}, vol. 33(3), pp. 41-52, 2012.

\bibitem{Wang08}
C. Wang, S. Mahadevan. Manifold alignment using
procrustes analysis. \emph{Intl.
Conf. on Machine Learning}, 2008, pp. 1120-1127.

\bibitem{Wang09}
C. Wang, S. Mahadevan. A general framework for
manifold alignment. \emph{AAAI Fall Symposium on Manifold
Learning and Its Applications}, 2009, pp. 53-58.

\bibitem{chnon}
F. J. Doyle, M. A. Henson. Nonlinear Systems Theory. M.A. Henson, D.E. Seborg (Eds), \emph{Nonlinear
Process Control}, Prentice Hall, 1997.

\bibitem{Isidori}
A. Isidori. \emph{Nonlinear Control Systems}, 3rd edition. Springer-Verlag London Limited, 1995.

\bibitem{Sontag}
E. Sontag, Y. Wang. A notion of input to output stability. \emph{European Control Conf.}, 1997, pp. 3862-3867.


\bibitem{angela_quad}
A. P. Schoellig, M. Hehn, S. Lupashin, R. D'Andrea. Feasibility of motion primitives for choreographed quadrocopter flight. \emph{American Control Conf.}, 2011, pp. 3843-3849. 

\bibitem{HS16}
M. K. Helwa, A. P. Schoellig. On the construction of safe controllable regions for affine systems with applications to robotics. \emph{IEEE Conf. on Decision and Control}, 2016, pp. 3000-3005. \emph{Available at ArXiv, arXiv:1610.01243 [cs.SY]}. 


%
%
%
%
%
%
%
%
%
%
%
%
%
%
%
%
%
%
%
%
%
%
%
%
%
%
%
%
%
%
%
%
%
%
%
%
%
%
%
%
%
%
%
%
%
%
%
%
%
%
%
%
%
%
%
%
%
%
%
%
%
%
%
%
%
%
%
%
%
%
%
%
%
%
%
%



\end{thebibliography}

\end{document}